%% file: main.tex
\documentclass{article}

\usepackage{microtype}
\usepackage{graphicx}
\usepackage{subfigure}
\usepackage{adjustbox}
\usepackage{booktabs} % for professional tables

\usepackage{hyperref}

\usepackage[accepted]{icml2026}

\usepackage{amsmath}
\usepackage{amssymb}
\usepackage{mathtools}
\usepackage{amsthm}
\usepackage{bbm}
\usepackage{xcolor}
\usepackage{soul}
\usepackage{multirow}
\usepackage{pifont}

\definecolor{Gray}{gray}{0.95}
\definecolor{Green}{HTML}{A4E28E} % Green
\definecolor{LightGreen}{HTML}{E0F6DE} % LightGreen
\definecolor{Blue}{HTML}{BFC0FF} % Blue
\definecolor{LightBlue}{HTML}{E7E6FF} % LightBlue
\definecolor{DarkRed}{HTML}{C00000}
\definecolor{DarkGreen}{HTML}{3B7D23}

\usepackage[capitalize,noabbrev]{cleveref}

\input{math_commands.tex}

\theoremstyle{plain}
\newtheorem{theorem}{Theorem}[section]
\newtheorem{proposition}[theorem]{Proposition}

\newtheorem{corollary}[theorem]{Corollary}
\theoremstyle{definition}
\newtheorem{definition}[theorem]{Definition}

\theoremstyle{remark}
\newtheorem{remark}[theorem]{Remark}

\usepackage[textsize=tiny]{todonotes}

\NewDocumentCommand{\weixiang}
{ mO{} }{\textcolor{orange}{\textsuperscript{\textit{Weixiang}}\textsf{\textbf{\small[#1]}}}}
\NewDocumentCommand{\zehong}
{ mO{} }{\textcolor{cyan}{\textsuperscript{\textit{Zehong}}\textsf{\textbf{\small[#1]}}}}

\begin{document}

\twocolumn[
    \icmltitle{Temporal Graph Pattern Machine}

    % It is OKAY to include author information, even for blind
    % submissions: the style file will automatically remove it for you
    % unless you've provided the [accepted] option to the icml2024
    % package.

    % List of affiliations: The first argument should be a (short)
    % identifier you will use later to specify author affiliations
    % Academic affiliations should list Department, University, City, Region, Country
    % Industry affiliations should list Company, City, Region, Country

    % You can specify symbols, otherwise they are numbered in order.
    % Ideally, you should not use this facility. Affiliations will be numbered
    % in order of appearance and this is the preferred way.
    \icmlsetsymbol{equal}{*}

    \begin{icmlauthorlist}
        \icmlauthor{Yijun Ma}{nd,equal}
        \icmlauthor{Zehong Wang}{nd,equal}
        \icmlauthor{Weixiang Sun}{nd}
        \icmlauthor{Yanfang Ye}{nd}
    \end{icmlauthorlist}

    \icmlaffiliation{nd}{University of Notre Dame}
    % \icmlaffiliation{uconn}{University of Connecticut}

    \icmlcorrespondingauthor{Zehong Wang}{zwang43@nd.edu}
    \icmlcorrespondingauthor{Yanfang Ye}{yye7@nd.edu}

    % You may provide any keywords that you
    % find helpful for describing your paper; these are used to populate
    % the "keywords" metadata in the PDF but will not be shown in the document
    \icmlkeywords{Machine Learning, ICML}

    \vskip 0.3in
]

% this must go after the closing bracket ] following \twocolumn[ ...

% This command actually creates the footnote in the first column
% listing the affiliations and the copyright notice.
% The command takes one argument, which is text to display at the start of the footnote.
% The \icmlEqualContribution command is standard text for equal contribution.
% Remove it (just {}) if you do not need this facility.

\printAffiliationsAndNotice{}  % leave blank if no need to mention equal contribution
% \printAffiliationsAndNotice{\icmlEqualContribution} % otherwise use the standard text.

% \begin{figure*}[!t]
%     \centering
%     \includegraphics[width=\linewidth]{fig/example.pdf}
%     \caption{The overview framework of GPM.}
%     \label{fig:example}
% \end{figure*}

% \input{table/example}

% \citep{example}

\input{section/abstract}
\input{section/introduction}
\input{section/methodology_icml}
\input{section/experiment_icml}

\input{section/related_work}
\input{section/conclusion}
% \input{section/impact_statement}

% In the unusual situation where you want a paper to appear in the
% references without citing it in the main text, use \nocite
% \nocite{langley00}

\bibliography{citation}
\bibliographystyle{icml2026}

%%%%%%%%%%%%%%%%%%%%%%%%%%%%%%%%%%%%%%%%%%%%%%%%%%%%%%%%%%%%%%%%%%%%%%%%%%%%%%%
%%%%%%%%%%%%%%%%%%%%%%%%%%%%%%%%%%%%%%%%%%%%%%%%%%%%%%%%%%%%%%%%%%%%%%%%%%%%%%%
% APPENDIX
%%%%%%%%%%%%%%%%%%%%%%%%%%%%%%%%%%%%%%%%%%%%%%%%%%%%%%%%%%%%%%%%%%%%%%%%%%%%%%%
%%%%%%%%%%%%%%%%%%%%%%%%%%%%%%%%%%%%%%%%%%%%%%%%%%%%%%%%%%%%%%%%%%%%%%%%%%%%%%%

\newpage
\appendix
\onecolumn
\input{section/appendix}

\end{document}

%% file: math_commands.tex
\usepackage{amsmath,amsfonts,bm}

\def\eqref#1{equation~\ref{#1}}
\def\1{\bm{1}}

\DeclareMathAlphabet{\mathsfit}{\encodingdefault}{\sfdefault}{m}{sl}
\SetMathAlphabet{\mathsfit}{bold}{\encodingdefault}{\sfdefault}{bx}{n}

%% file: section/abstract.tex
\begin{abstract}
    Temporal graph learning is pivotal for deciphering dynamic systems, where the core challenge lies in explicitly modeling the underlying evolving patterns that govern network transformation. However, prevailing methods are predominantly task-centric and rely on restrictive assumptions---such as short-term dependency modeling, static neighborhood semantics, and retrospective time usage. These constraints hinder the discovery of transferable temporal evolution mechanisms. To address this, we propose the Temporal Graph Pattern Machine (TGPM), a foundation framework that shifts the focus toward directly learning generalized evolving patterns.
    TGPM conceptualizes each interaction as an \textit{interaction patch} synthesized via temporally-biased random walks, thereby capturing multi-scale structural semantics and long-range dependencies that extend beyond immediate neighborhoods. These patches are processed by a Transformer-based backbone designed to capture global temporal regularities while adapting to context-specific interaction dynamics. To further empower the model, we introduce a suite of self-supervised pre-training tasks---specifically masked token modeling and next-time prediction---to explicitly encode the fundamental laws of network evolution. Extensive experiments on temporal link prediction and temporal node classification show that TGPM consistently ranks among the top-performing methods, demonstrating exceptional cross-domain transferability. Our code has been released in \url{https://github.com/antman9914/TGPM}.
\end{abstract}

%% file: section/introduction.tex
\section{Introduction}
\label{sec:intro}

Real-world complex systems are inherently temporal: system states evolve continuously as events unfold, and future behaviors emerge as a consequence of historical dynamics~\cite{kossinets2006empirical,cai2024survey}. In such systems, events are neither independent nor random; instead, they are jointly shaped by historical context, interactions among entities, and temporal factors such as periodicity, burstiness, and other time-dependent conditions ~\cite{liguori2025131191}. As a result, temporal systems exhibit structured evolution patterns that depend on both historical trajectories and relational contexts ~\cite{holme201297}. 
Temporal graphs ~\cite{Michail03072016} provide a principled representation for modeling these systems by encoding entities as nodes and time-stamped interactions as temporal edges, thereby unifying relational structure with temporal dynamics. Modeling how temporal graphs evolve, and how future events are generated from historical and structural context, has consequently become a central problem in machine learning and data mining~\cite{longa2023graph}.

Although temporal graphs provide a principled modeling method for temporal systems, learning from them remains fundamentally challenging. The difficulty arises not primarily from data scale \cite{Yanping2025196323} or computational complexity \cite{wang2024tgl}, but from uncovering the underlying mechanisms that govern how temporal systems evolve~\cite{pan2025crossdyg}. 
To accurately understand and predict future behaviors in a temporal environment, models must move beyond optimizing predictive performance for individual downstream tasks and instead capture how events are generated through the interplay of historical context, relational structure, and temporal dynamics~\cite{huang2025crosslink}. From this perspective, temporal graph learning can be viewed as a form of \textit{mechanism learning}, whose objective is to infer the latent processes driving system evolution rather than to fit task-specific correlations.

Despite their practical effectiveness, most existing temporal graph learning methods typically rely on several modeling assumptions that are fundamentally misaligned with the objective of this mechanism.
\textit{(1) Static neighborhood semantics assumption.}
Most methods model temporal graph evolution primarily through one-hop neighborhood dynamics~\cite{yu2023towards,wu2024feasibility,tian2024freedyg}
% , which are typically encoded using temporal encodings and variations in sampled interactions
. However, neighborhood semantics in temporal systems are inherently non-stationary ~\cite{layne2023role}: the functional role and behavioral patterns of the same neighbor may change substantially as the relational structure evolves. Compressing such semantic dynamics into fixed temporal embeddings and sampling procedures therefore constrains a model’s ability to learn evolution patterns~\cite{feng2026survey,wang2021time}
that generalize across different graph instances and temporal regimes.
\textit{(2) Short-term dependency assumption.}
Many approaches assume that recent neighbor interactions are sufficient to characterize a node’s current state, and consequently restrict representation learning to local time windows or recent-neighbor sampling~\cite{wanginductive,xu2020Inductive,congwe,yu2023towards}. Although computationally efficient, this design choice systematically biases models toward short-term signals ~\cite{wang2021time} and limits their ability to capture long-term dependencies that often govern the evolution of real-world systems.
\textit{(3) Retrospective temporal modeling assumption.}
In most existing methods, temporal modeling is predominantly retrospective~\cite{yu2023towards,wu2024feasibility}, which means that temporal information is incorporated only through post-hoc conditioning on past events, where temporal signals are treated as auxiliary annotations of historical events. Such temporal annotations are exploited via time decay~\cite{nguyen2018continuous} and attention weighting~\cite{xu2020Inductive} to modulate the influence of past events, but they are not treated as an explicit modeling target.
Consequently, while these models can identify which historical events are relevant, they often fail to accurately characterize when future events will occur or to capture the temporal structures underlying system evolution.
Together, these assumptions suggest that the central challenge of temporal graph learning lies not in improving performance on individual prediction tasks, but in understanding the underlying processes by which temporal systems evolve~\cite{huang2025crosslink,pan2025crossdyg}.

To overcome these challenges, effective temporal graph modeling requires explicitly learning the evolving patterns that govern system dynamics. Such patterns correspond to generalizable generative mechanisms that characterize how historical dependencies, structural conditions, and temporal regularities jointly shape future events \cite{huang2025crosslink}. Importantly, these evolution patterns are not tied to specific prediction tasks or individual graph instances; instead, they capture system-level regularities of temporal dynamics that recur across different domains and time horizons.

Motivated by these insights, we propose \textbf{Temporal Graph Pattern Machine (TGPM)}, a pattern-centric modeling framework that explicitly treats evolving patterns as the primary modeling target. Instead of learning task-specific representations tied to individual nodes or events, TGPM aims to capture generalizable temporal evolution mechanisms from temporal graphs.
To this end, we design a modeling paradigm in which temporal contexts are represented as sequences of evolving patterns that jointly encode structural relationships, historical dependencies, and temporal information (C1). TGPM is trained using self-supervised objectives that encourage the model to capture temporal dependencies at multiple time scales (C2) and to explicitly reason about future event timing (C3). As a result, the learned representations support robust generalization across different tasks, time periods, and graph domains.

Our main contributions are summarized as follows:
(1) \textbf{Pattern-centric modeling paradigm.} We propose a pattern-centric perspective for temporal graph learning that explicitly targets the modeling of evolving patterns governing system dynamics. This paradigm aims to capture generalizable temporal evolution regularities across different graphs, tasks, and time periods; (2) \textbf{Temporal Graph Pattern Machine (TGPM).} We introduce TGPM, a new temporal graph modeling framework that represents temporal contexts as sequences of interaction patterns constructed via temporally biased random walks and encodes them using a Transformer-based architecture. TGPM is trained with self-supervised objectives that jointly model multi-scale temporal dependencies and future event timing, enabling the learning of generalizable representations; (3) \textbf{Extensive empirical evaluation.} We conduct comprehensive experiments on temporal link prediction and temporal node classification over multiple real-world temporal graph benchmarks. The results demonstrate that TGPM consistently achieves the best overall average ranking and exhibits robust generalization across different scenarios.
% \begin{itemize}
%     \item \textbf{Pattern-centric modeling paradigm.} We propose a pattern-centric perspective for temporal graph learning that explicitly targets the modeling of evolving patterns governing system dynamics. This paradigm aims to capture generalizable temporal evolution regularities across different graphs, tasks, and time periods.
    
%     \item \textbf{Temporal Graph Pattern Machine (TGPM).} We introduce TGPM, a new temporal graph modeling framework that represents temporal contexts as sequences of interaction patterns constructed via temporally biased random walks and encodes them using a Transformer-based architecture. TGPM is trained with self-supervised objectives that jointly model multi-scale temporal dependencies and future event timing, enabling the learning of generalizable representations.
    
%     \item \textbf{Extensive empirical evaluation.} We conduct comprehensive experiments on temporal link prediction and temporal node classification over multiple real-world temporal graph benchmarks. The results demonstrate that TGPM consistently outperforms strong baselines and exhibits robust generalization across different scenarios.
% \end{itemize}

%% file: section/methodology_icml.tex
\begin{figure*}[!t]
    \centering
    \includegraphics[width=\linewidth]{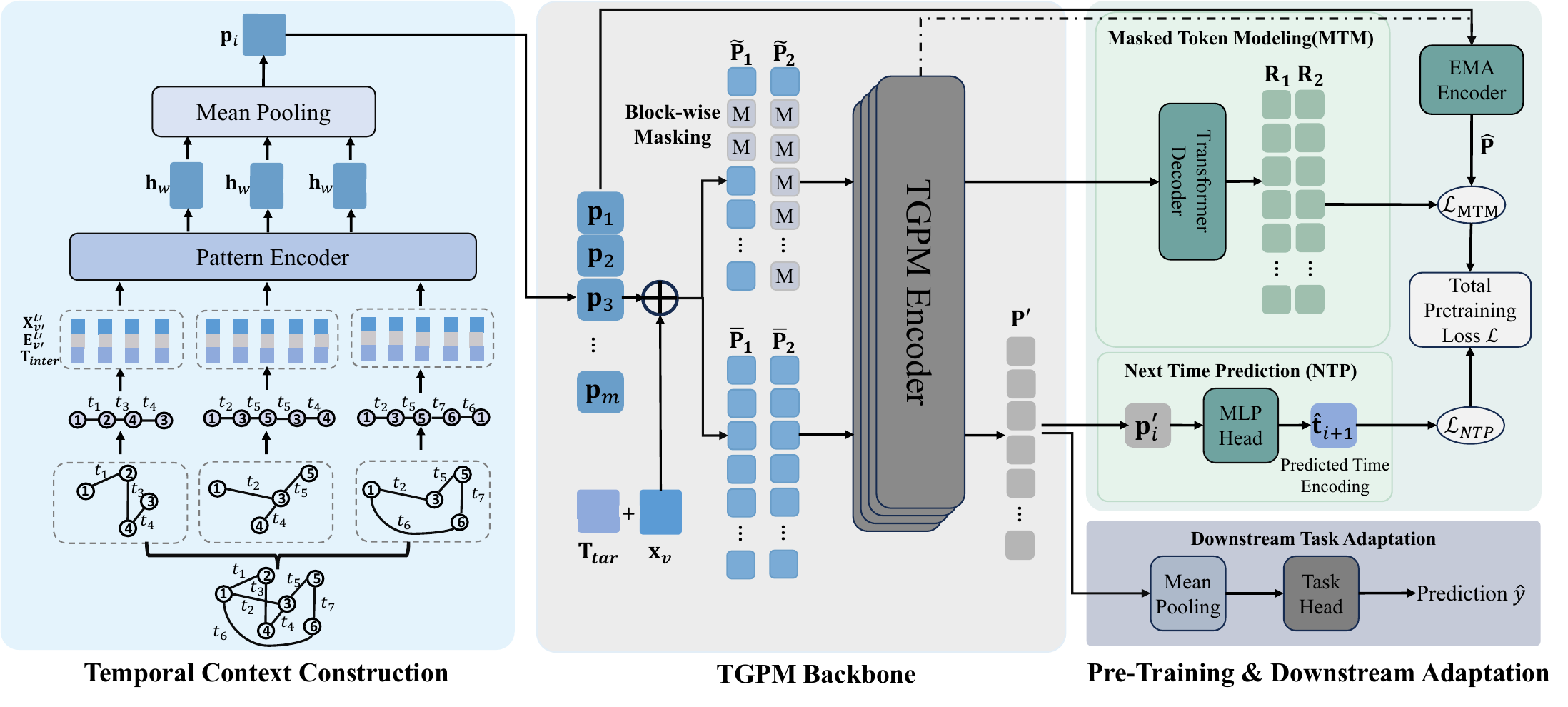}
    \caption{\textbf{Overview of TGPM}. (a) Temporal context of an interaction is represented by an interaction patch aggregated from a set of temporally biased random walks. (b) Transformer-based TGPM encoder adapt local structural and temporal semantics to target-specific temporal context, capturing global temporal regularities. (c) MTM and NTP pre-training enforce TGPM to learn from multi-scale temporal dynamics and the temporal rhythms of evolving patterns. }
    \label{fig:model}
    \vspace{-10pt}
\end{figure*}

\section{Temporal Graph Pattern Machine}

A temporal graph is denoted as $\mathcal{G} = (\mathcal{V}, \mathcal{E}, \mathbf{X}, \mathbf{E})$, where $\mathcal{V}$ is the set of nodes with $|\mathcal{V}| = N$, and $\mathcal{E}$ is a sequence of chronologically ordered interactions $\{(u_1, v_1, t_1), (u_2, v_2, t_2), \ldots, (u_E, v_E, t_E)\}$ satisfying $0 \le t_1 \le t_2 \le \ldots \le t_E$. For each interaction, $u_i, v_i \in \mathcal{V}$ denote the source and destination nodes at timestamp $t_i$. $\mathbf{X}$ and $\mathbf{E}$ are the node and edge feature matrices, where each node $v \in \mathcal{V}$ is associated with a feature vector $\mathbf{x} \in \mathbb{R}^{d_n}$ and each edge $e \in \mathcal{E}$ with a feature vector $\mathbf{e} \in \mathbb{R}^{d_e}$.

As illustrated in Figure~\ref{fig:model}, TGPM models temporal contexts as sequences of \emph{interaction patches}, each derived from evolving local substructures to capture dynamic neighborhood semantics. These interaction patches are then encoded and contextualized to form temporal representations, which are further optimized through self-supervised pre-training objectives, including masked token modeling and next time prediction, to capture multi-scale temporal dependencies and future event dynamics.

\subsection{Temporal Context Representation}
\label{sec:context}

In common practice of temporal context construction~\cite{rossi2020temporal,yu2023towards}
, given a target timestamp $t$, the temporal context of a target node $v$ is represented as a sequence of first-hop interactions, denoted by
$\mathcal{S}_v^t=\{(v, v', t') \mid t'<t\}$.
Such representations are inherently node-centric and rely on temporal encodings to summarize past interactions. However, the semantic content of an interaction is not solely determined by the participating nodes and timestamps, but also by the evolving structural context in which it occurs \cite{layne2023role}. Compressing this information into node-level temporal features therefore leads to substantial loss of structural and semantic dynamics.

To overcome static neighborhood semantics assumption, we propose a pattern-centric modeling paradigm in which each interaction is represented by an \emph{interaction patch}. Considering that temporal substructures can naturally reflect how local graph substructures and historical interactions co-evolve over time, we design interaction patches to summarize the evolving structural and temporal context associated with a specific interaction by aggregating multiple temporally contextualized substructures.
Individual substructure captures only partial view of the evolving semantics, while aggregating multiple complementary substructures into an interaction patch can provide more stable and semantically meaningful representations of how interaction semantics evolve over time.
Interaction patch serves as the basic semantic unit for modeling temporal contexts in TGPM and enables the representation of interactions beyond simple node-level histories.

\noindent\textbf{Sampling Substructures via Temporal Random Walks.}
A key challenge in constructing interaction patches lies in how to systematically tokenize evolving relational structures into sequences of substructures. Direct enumeration or matching of subgraphs is computationally prohibitive~\cite{sun2012match,zhao2025controllable} and requires a predefined substructure vocabulary~\cite{wang2024gft}. 
However, a suitable tokenization mechanism should remain efficient and scalable as temporal graphs evolve, and meanwhile preserve diverse substructures and temporal dependencies. Temporal random walk sampling naturally satisfies these principles by approximating substructure matching through scalable time-aware stochastic traversal. As a result, we propose to construct interaction patches using \emph{temporally biased random walks}.
These walks provide a scalable approximation to atomic evolution patterns while bypassing the need for explicit substructure enumeration \cite{nguyen2018continuous}. Moreover, they naturally encode both historical dependencies and structural relations within a unified sequential representation.

For each interaction $(v, v', t')$, we generate a set of contextualized walks rooted at node $v'$. To avoid data leakage, ensure pattern diversity, and preserve contextual dependency, the sampling process follows two principles:
(1) \emph{Retrospective sampling}: all sampled edges are considered as undirected edges and must have timestamps earlier than $t'$; and
(2) \emph{Temporal recency bias}: edges closer to $t'$ are assigned higher sampling priority, following the time-decay assumption widely adopted in dynamic processes~\cite{hawkes1971spectra,kalman1960new}.
Unlike traditional chronologically monotonic temporal walks~\cite{wanginductive,nguyen2018continuous}, we relax strict causal ordering and allow more complex temporal dependencies.

Formally, let $\mathcal{T}: \mathcal{E} \rightarrow \mathbb{R}^+$ denote the edge timestamp mapping function. Given an anchor interaction $(v, v', t')$, a temporally biased random walk $w$ of length $L$ is defined as a node sequence
$w=(v_0,v_1,\dots,v_L)$ with $v_0=v'$,
generated by the following Markov process:
\begin{equation}
    \eta(u,v) =
 \begin{cases}
 \exp(t'-\mathcal{T}(u,v)), & (u,v)\in\mathcal{E},\ \mathcal{T}(u,v)<t' \\
     0,                         & \text{otherwise}
  \end{cases} 
\end{equation}
\begin{equation}
    P(v_{i+1}\mid v_0, v_1, \dots, v_i) = \frac{\eta(v_i, v_{i+1})}{\sum_{u\in \mathcal{V}}\eta(v_i, u)}.
\end{equation}
This transition rule prioritizes temporally close edges while allowing non-monotonic traversal of historical interactions.

\begin{proposition}
\label{prop:causal_walk}
Given a temporal random walk $w=(v_0,v_1,\dots,v_L)$ with $\mathcal{T}(v_i,v_{i+1})<t', i=0,1,\dots,L-1$ and a causal walk $w_c = (u_0,u_1,\dots,u_L)$ with $\mathcal{T}(u_{i-1},u_i) < \mathcal{T}(u_i, u_{i+1}) < t', i=1,2,\dots,L-1$, $w$ 
can capture more complex evolving semantics than $w_c$.
\end{proposition}

It is proved in Appendix \ref{sec:prop1} that temporally non-monotonic random walks are more expressive than causally monotonic walks, able to encode temporal structures that are fundamentally inaccessible under causal constraints. This justify Proposition \ref{prop:causal_walk} and our proposed sampling principles.

\noindent\textbf{Constructing Patch Embeddings via Aggregating Substructures.}
For each interaction $(v, v', t')$, we sample $k$ temporally biased random walks (i.e., substructures) to construct its contextualized evolution patterns. For a walk $w$, we retrieve the corresponding edge sequence
$w_e=((v, v'), (v', v_1), \dots, (v_{L-1}, v_L))$
and timestamp sequence
$w_t=(t', \mathcal{T}(v', v_1), \dots, \mathcal{T}(v_{L-1}, v_L))$.
Based on $w$ and $w_e$, we obtain the node and edge feature matrices
$\mathbf{X}_{v'}^{t'}\in\mathbb{R}^{L\times d_n}$ and
$\mathbf{E}_{v'}^{t'}\in\mathbb{R}^{L\times d_e}$ from $\mathbf{X}$ and $\mathbf{E}$, respectively.
To encode temporal information, we adopt sinusoidal time-interval embeddings~\cite{xu2020Inductive,yu2023towards}. Using $t'$ as the anchor time, each timestamp $t_i$ in $w_t$ is converted to an interaction-relative interval $\Delta t=t'-t_i$ and encoded by time encoder $T_{enc}$ as
\begin{equation}
\begin{aligned}
T_{enc}(\Delta t)
&= \sqrt{\frac{1}{d_t}}
\bigl[
\cos(\omega_1 \Delta t), \sin(\omega_1 \Delta t), \dots, \\
&\qquad \cos(\omega_{d_t}\Delta t), \sin(\omega_{d_t}\Delta t)
\bigr],
\end{aligned}
\end{equation}
where $d_t$ is the encoding dimension and $\{\omega_i\}$ are trainable parameters.
The resulting matrix is denoted by $\mathbf{T}_{inter} \in \mathbb{R}^{L \times d_t}$. 
The concatenated representation
$\bar{\mathbf{H}}_w=[\mathbf{X}_{v'}^{t'} \| \mathbf{E}_{v'}^{t'} \| \mathbf{T}_{inter}]$
is projected via a linear layer and encoded by a Transformer encoder $f$:
\begin{equation}
    \mathbf{h}_w = f(\mathbf{H}_w), \quad \mathbf{H}_w = \mathbf{W}\bar{\mathbf{H}}_w + \mathbf{b},
\end{equation}
where $\mathbf{h}_w$ denotes the pattern embedding of walk $w$.

Finally, the interaction patch embedding $\mathbf{p}$ for $(v,v',t')$ is obtained by aggregating the embeddings of its $k$ sampled walks via mean pooling:
$\mathbf{p}=\frac{1}{k}\sum_{w}\mathbf{h}_w$.
This interaction patch serves as a pattern-centric representation that summarizes the evolving structural and temporal semantics of the interaction and forms the basic input token for the TGPM backbone encoder.

\subsection{TGPM Backbone}

Given a target node $v$ at time $t$, its temporal context is represented as a sequence of interaction patch embeddings
$\mathbf{P}=[\mathbf{p}_1,\mathbf{p}_2,\dots,\mathbf{p}_m]$,
where $m=|\mathcal{S}_v^t|$ and each $\mathbf{p}_i$ corresponds to the interaction patch of $(v,v'_i,t'_i)\in\mathcal{S}_v^t$.
While interaction patches capture local structural and temporal semantics, they are invariant to absolute time shifts and lack target-specific temporal localization.
To contextualize these patches, we adopt a Transformer backbone that models their temporal dependencies and relevance to the target node.
We construct an \emph{interaction token sequence}
\[
\bar{\mathbf{P}}=[\mathbf{P}\,\|\,\mathbf{T}_{tar}\,\|\,\mathbf{x}_v],
\]
where $\mathbf{T}_{tar}\in\mathbb{R}^{m\times d_t}$ denotes the target-relative time interval encodings derived by $T_{enc}(\cdot)$ with $\Delta t=t-t'_i$ as input, and $\mathbf{x}_v$ is the feature vector of node $v$. 

Given $\bar{\mathbf{P}}$, a Transformer layer computes
\begin{equation}
    \mathbf{Q}=\bar{\mathbf{P}}\mathbf{W}_Q,\quad
    \mathbf{K}=\bar{\mathbf{P}}\mathbf{W}_K,\quad
    \mathbf{V}=\bar{\mathbf{P}}\mathbf{W}_V,
\end{equation}
\begin{equation}
    \mathrm{Attn}(\bar{\mathbf{P}})=\mathrm{softmax}\!\left(\frac{\mathbf{Q}\mathbf{K}^T}{\sqrt{d_{out}}}\right)\mathbf{V},
\end{equation}
\begin{equation}
    \mathbf{P}'=\mathrm{FFN}\big(\bar{\mathbf{P}}+\mathrm{Attn}(\bar{\mathbf{P}})\big),
\end{equation}
where $\mathbf{W}_Q,\mathbf{W}_K,\mathbf{W}_V$ are trainable matrices and $d_{out}$ is the output dimension.
We employ multi-head attention~\cite{vaswani2017attention} and stack multiple layers.
The output of the final layer is
\(
\mathbf{P}'=[\mathbf{p}'_1,\mathbf{p}'_2,\dots,\mathbf{p}'_m],
\)
which represents the contextual embeddings of interaction patches and is used for pre-training and downstream tasks.

\subsection{TGPM Pre-Training}

Although the proposed pattern-centric representation alleviates the limitations of node-centric temporal modeling, naive task-centric training still tends to over-emphasize short-term correlations and treats time only implicitly. This behavior hinders the learning of generalizable temporal evolution mechanisms. To address these issues, we adopt self-supervised pre-training to decouple representation learning from downstream supervision and to explicitly model temporal dependencies at multiple time scales.
Specifically, we design two complementary objectives:
(1) \emph{Masked Token Modeling (MTM)} to capture semantic dependencies among interaction patterns across different temporal ranges and go beyond short-term correlations;
(2) \emph{Next Time Prediction (NTP)} to explicitly model when future interactions occur and model temporal dynamics from both retrospective and prospective perspectives.
Together, they encourage the model to learn both \emph{what} evolves and \emph{when} it evolves.

\noindent\textbf{Masked Token Modeling (MTM) for Learning Long-/Short-Horizon Dependency.}
Given an interaction token sequence $\bar{\mathbf{P}}=[\bar{\mathbf{p}}_1,\bar{\mathbf{p}}_2,\dots,\bar{\mathbf{p}}_m]$, MTM randomly masks blocks of consecutive tokens and trains the model to reconstruct the masked tokens from the visible context. This follows the conditional factorization principle of masked language modeling
$p(\bar{\mathbf{p}}_1,\dots,\bar{\mathbf{p}}_m)
=
p(\bar{\mathbf{p}}_{/M})\prod_{i\in M} p(\bar{\mathbf{p}}_i \mid \bar{\mathbf{p}}_{/M})$~\cite{devlin2019bert},
where $M$ denotes the masked positions and $\bar{\mathbf{p}}_{/M}$ the visible tokens.

To encourage dependency learning at multiple temporal scales (e.g., short-term or long-term), we adopt block-wise masking. Given a masking budget $n_m$ and the number of blocks $b\ge1$, we sample $b$ non-overlapping blocks of size $\lceil n_m / b \rceil$. Larger blocks require long-range dependency modeling, while smaller blocks emphasize short-term dependency modeling. Formally, we define the set of masked position $\mathcal{M}$ as
\begin{equation}
    \mathcal{M}=\bigcup\limits_{i=1}^b\{s_i,s_i+1,\dots,s_i+\lceil n_m / b \rceil - 1\}
\end{equation}
where $s_i$ is the index of $\bar{\mathbf{p}}_i$. To define reconstruction targets, we employ an exponential moving average encoder $f_{\mathrm{EMA}}$ to generate stable semantic targets
$\hat{\mathbf{p}}_i = f_{\mathrm{EMA}}(\{w \mid v_0=v'_i\})$.
The visible tokens are augmented with learnable mask tokens to form $\tilde{\mathbf{P}}$, which is passed through the TGPM encoder and a Transformer decoder to obtain reconstructed token $\mathbf{r}_i$ in SimMIM fashion~\cite{xie2022simmim}.
The loss for block size $b$ is defined as
\begin{equation}
\mathcal{L}_{\mathrm{MTM}}(b)
=
\frac{1}{m}
\sum_{i=1}^{m}
\mathrm{is\_masked}(\bar{\mathbf{p}}_i)
\cdot
\|\mathbf{r}_i - \mathrm{sg}[\hat{\mathbf{p}}_i]\|_2^2 ,
\end{equation}
where $\mathrm{is\_masked}(\cdot)$ is an indicator function for masked position, and $\mathrm{sg}[\cdot]$ denotes the stop-gradient operator. The final MTM objective aggregates losses over a predefined set of block sizes $B$:
\begin{equation}
\mathcal{L}_{\mathrm{MTM}} = \sum_{b\in B} \mathcal{L}_{\mathrm{MTM}}(b).
\end{equation}

\begin{proposition}
    \label{prop:mtm} 
    Under block-wise masked token modeling with masking block size $b$, the optimal reconstruction must exploit contextual dependencies beyond certain dependency range $r$.
\end{proposition}

We justify Proposition \ref{prop:mtm} in Appendix \ref{sec:prop2} that block-wise masking is not merely a heuristic design, but an information-theoretic necessity for enforcing multi-scale temporal dependency learning. The masking block size directly controls the minimum temporal horizon that the model is forced to reason over. 

\noindent\textbf{Next Time Prediction (NTP) for Prospective Temporal Modeling.} 
While MTM focuses on modeling semantic dependencies among interaction patterns, it does not explicitly constrain the temporal dimension. To model temporal evolution directly, we introduce the NTP task.
Given the embeddings from the TGPM encoder
$\mathbf{P}'=[\mathbf{p}'_1,\dots,\mathbf{p}'_m]$,
a two-layer MLP head $f_{\mathrm{NTP}}$ predicts the target-relative time interval encoding of the next interaction:
\begin{equation}
p(\mathbf{t}_1,\dots,\mathbf{t}_m)
=
\prod_{i=1}^{m}
p(\mathbf{t}_i \mid \mathbf{t}_{<i}, \bar{\mathbf{p}}_{<i}),
\end{equation}
where $\mathbf{t}_i$ is the encoding corresponding to $\mathbf{p}_i$.
The training objective is
\begin{equation}
\mathcal{L}_{\mathrm{NTP}}
=
\frac{1}{m-1}
\sum_{i=1}^{m-1}
\| f_{\mathrm{NTP}}(\mathbf{p}'_i) - \mathbf{t}_{i+1} \|.
\end{equation}

The NTP encourages TGPM to infer when the next interaction is likely to occur conditioned on the observed patterns. Specifically, it allows TGPM to encode the temporal granularity of evolution, such as the typical gaps between successive interaction patterns, as well as frequency-related signals, reflecting how often interactions involving certain evolving patterns occur. Thus, the model learns to associate different evolving patterns with distinct temporal rhythms.

\noindent\textbf{Joint training and downstream usage.}
The overall pre-training objective is  $\mathcal{L} = \mathcal{L}_{\mathrm{MTM}} + \mathcal{L}_{\mathrm{NTP}}$.
After pre-training, the decoder is discarded and a task-specific prediction head is attached to the TGPM encoder. For downstream tasks, we apply mean pooling over the contextualized sequence $\mathbf{P}'$ and compute predictions as
\(
\hat{\mathbf{y}}
=
\mathrm{Head}\!\left(\frac{1}{m}\sum_{i=1}^{m}\mathbf{p}'_i\right)
\).

%% file: section/experiment_icml.tex
\input{table/main_result}

% \begin{figure*}[!t]
% \centering
% \begin{minipage}[!t]{0.55\textwidth}
% \centering
% \captionof{table}{\textbf{Performance and average ranking (A.R.) comparison over temporal node classification.} The best and sub-best results are highlighted in \textbf{Boldface} and \underline{Underline}. TGPM achieves the strongest A.R..
% }
\begin{table}
\centering
\caption{\textbf{Performance and average ranking (A.R.) comparison over temporal node classification.} The best and sub-best results are highlighted in \textbf{Boldface} and \underline{Underline}. TGPM achieves the strongest A.R..
}
\label{tab:nc}
  \resizebox{0.9\columnwidth}{!}{
    \begin{tabular}{l|cc|c}
    \toprule
    \multicolumn{1}{c|}{\textbf{Methods}} & \textbf{Wikipedia} & \textbf{Reddit} & \textbf{A.R.} \\
    \midrule
    TGAT  & 80.03±1.43 & 50.96±3.85 & 6.5 \\
    GraphMixer & 85.47±0.79 & 51.98±1.78 & 4.5 \\
    DyGFormer & 77.93±11.79 & 53.35±1.78 & 6.0 \\
    \midrule
    PT-DGNN & 81.61±1.29 & 49.32±1.82 & 6.0 \\
    DDGCL & 78.04±3.95 & 53.14±3.24 & 6.0 \\
    CPDG  & 80.47±0.44 & 54.99±2.40 & 4.0 \\
    \midrule
    TGPM w/o Pretrain & \underline{86.83±1.08} & \underline{55.21±2.17} & \underline{2.0}  \\
    TGPM  & \textbf{87.40±0.70} & \textbf{55.54±1.22} & \textbf{1.0} \\
    \bottomrule
    \end{tabular}%
    }
\end{table}
% \end{minipage}
% \hfill
% \begin{minipage}[!t]{0.4\textwidth}
% \includegraphics[width=0.9\linewidth]{fig/curve.png}
% \caption{\textbf{Convergence curve comparison between training from scratch and fine-tuning the pre-trained model. }
% }
% \label{fig:converge}

% \end{minipage}
% \vspace{-10pt}
% \end{figure*}

\section{Experiments}
\label{sec:experiment}

\subsection{Experimental Settings}

We evaluate the effectiveness of our proposed TGPM on two downstream tasks. For temporal link prediction, we adopt three datasets from DTGB~\cite{zhang2024dtgb}, including Enron, ICEWS1819 and Googlemap CT. Following prior work, evaluation is conducted under both transductive and inductive settings~\cite{yu2023towards,congwe}. The transductive setting aims to predict future links between nodes within training set, and the inductive setting predicts future links on unseen nodes. 
For each link, one negative sample is generated for classification.
For temporal node classification, we adopt two widely used datasets, Wikipedia and Reddit~\cite{kumar2019jodie}. 
These datasets are featured by rich attributes, spanning various domains. We adopt the same dataset split strategy for all these datasets, which is 70/15/15 train/val/test chronological split. ROC-AUC score is used as the evaluation metric for both downstream tasks. 

We compare our TGPM with supervised methods (TGAT~\cite{xu2020Inductive}, GraphMixer~\cite{congwe}, DyGFormer~\cite{yu2023towards}) and self-supervised learning methods (PT-DGNN~\cite{chen2022679}, DDGCL~\cite{tian2021ddgcl}, CPDG~\cite{bei2023cpdg}) targeting temporal graphs. For TGPM, we report the results both before (TGPM w/o Pretrain) and after pre-training (TGPM). More details are provided in Appendix \ref{sec:implementation}.

\subsection{Main Results}

\textbf{Temporal Link Prediction.} Table \ref{tab:main} summarizes temporal link prediction results under both transductive and inductive settings. Pre-training is conducted on each individual dataset. Overall, TGPM achieves the best average ranking of 3.0 among all selected methods, with best and sub-best performance in most datasets under both transductive and inductive setting. Notably, on Googlemap CT with more temporal dynamics, our proposed pre-training strategies can bring significant performance gain (80.51 vs. 78.55), highlighting its ability to capture complex evolving patterns. In contrast, on datasets with significant temporal burstiness, TGPM already performs well without pre-training due to sufficient temporal context information, but our pre-training strategies are much less beneficial and may even lead to trivial solution, reflecting the inadequacy of our method in handling temporal burstiness. 

\textbf{Temporal Node Classification.} Table \ref{tab:nc} summarizes temporal node classification results. Pretrained TGPM and TGPM trained from scratch achieve the best and sub-best performance on both datasets. Given that both Wikipedia and Reddit are featured by rich temporal dynamics, the consistent performance gain demonstrate the strength of TGPM in capturing complex evolving patterns. 

\input{table/transfer}

\begin{figure*}[!t]
\centering
\begin{minipage}[!t]{0.48\textwidth}
\centering
\captionof{table}{\textbf{Cross-task transferability of TGPM compared with in-domain pre-training.} Transferred TGPM achieves competitive and even superior temporal node classification performance compared to in-domain pre-trained model.
}
\label{tab:cross_task}
  \resizebox{0.9\columnwidth}{!}{
    \begin{tabular}{l|rr|r}
    \toprule
    \multicolumn{1}{c|}{\textbf{Methods}} & \multicolumn{1}{c}{\textbf{Wikipedia}} & \multicolumn{1}{c|}{\textbf{Reddit}} & \multicolumn{1}{c}{\textbf{A.R.}} \\
    \midrule
    PT-DGNN & 81.61±1.29 & 49.32±1.82 & 5.0 \\
    DDGCL & 78.04±3.95 & 53.14±3.24 & 5.5 \\
    CPDG  & 80.47±0.44 & 54.99±2.40 & 4.0 \\
    \midrule
    TGPM  & \textbf{87.40±0.70} & \underline{55.54±1.22} & \textbf{1.5} \\
    \midrule
    TGPM (ICEWS) & 85.64±0.80 & 54.99±2.78 & \underline{3.0} \\
    TGPM (Googlemap) & \underline{86.76±0.91} & \textbf{61.16±2.59} & \textbf{1.5} \\
    \bottomrule
    \end{tabular}%
    }
\end{minipage}
\hfill
\begin{minipage}[!t]{0.51\textwidth}
\centering
\captionof{table}{\textbf{Ablation of model components and alternative solutions in TGPM on Googlemap CT and Wikipedia.}
}
\label{tab:ablation}%
  \resizebox{0.9\columnwidth}{!}{
    \begin{tabular}{l|cc|c}
    \toprule
    \multicolumn{1}{c|}{\multirow{2}[4]{*}{\textbf{Methods}}} & \multicolumn{2}{c|}{\textbf{Googlemap CT}} & \multirow{2}[4]{*}{\textbf{Wikipedia}} \\
\cmidrule{2-3}          & \textbf{Transductive} & \textbf{Inductive} &  \\
    \midrule
    Full Method & \textbf{80.51} & \textbf{74.32} & \textbf{87.40} \\
    \midrule
    TGPM w/o ntp & 79.59 & 73.25 & 84.63 \\
    TGPM w/o ltm & 80.44 & 74.28 & 86.47 \\
    TGPM w/o stm & 79.31 & 72.44 & 84.38 \\
    \midrule
    TGPM w/ cp & 79.58 & 73.07 & 87.00 \\
    TGPM w/ rm & 79.39 & 72.62 & 87.25 \\
    \bottomrule
    \end{tabular}%
    }
\end{minipage}
\end{figure*}

\textbf{Cross-Domain Transferability.}
We evaluate the cross-domain transferability of TGPM in Table \ref{tab:cross}, measuring generalization under distribution shifts across domains. In this setting, we train models on a source graph and evaluate on target graphs without further adaptation. We adopt ICEWS1819 and Googlemap CT as source graphs respectively and treat the remaining graphs as target graphs. Self-supervised learning baselines (PT-DGNN, DDGCL, CPDG) are adopted for comparison. TGPM achieves the best transfer results, showing significant positive gains over all target graphs. We attribute this to TGPM's ability to learn transferable evolving patterns, enabling generalization across significantly different domains, while time-shifting and structural proximity based self-supervised learning remains sensitive to subtle shifts of temporal dynamics. 

\textbf{Cross-Task Transferability.} We measure the cross-task transferability of TGPM, evaluating whether the learned evolving patterns in TGPM can be effectively generalized to diverse feature spaces and downstream tasks. In this setting, we pre-train TGPM on link prediction datasets ICEWS1819 and Googlemap CT respectively, and further adapt to temporal node classification via fine-tuning. We report the in-domain pre-trained results of all self-supervised methods and cross-task fine-tuned results of TGPM in Table \ref{tab:cross_task}, where TGPM (ICEWS) and TGPM (Googlemap) denote the respective source datasets. TGPM transferred from both temporal link prediction datasets achieves competitive and even superior performance on both Wikipedia and Reddit, demonstrating TGPM's ability in capturing generalizable evolving patterns regardless of feature and temporal heterogeneity. Besides, TGPM pre-trained on Googlemap CT shows stronger cross-task generalizability, indicating that TGPM can better leverage its advantages on datasets with complex temporal dynamics.

\subsection{Discussion}

% \begin{figure*}[!t]
% \centering
% \begin{minipage}[!t]{0.55\textwidth}
% \centering
% \includegraphics[width=0.9\linewidth]{fig/scalability.png}
% \caption{\textbf{Parameter Scaling Analysis.} Increasing model parameters can consistently enhance performance on transductive settings, showing the great potential of TGPM acting as the scalable backbone in temporal graph learning. 
% }
% \label{fig:scale}
% \end{minipage}
% \hfill
% \begin{minipage}[!t]{0.44\textwidth}
% \includegraphics[width=0.9\linewidth]{fig/efficiency_plot.pdf}
% \caption{\textbf{Computational Efficiency Analysis.} Sizes of nodes indicate the peak GPU memory usage. TGPM maintains competitive training efficiency, significantly more efficient than self-supervised baselines.
% }
% \label{fig:efficiency}
% \end{minipage}
% \vspace{-10pt}
% \end{figure*}

\begin{figure}[!t]
    \centering
    \subfigure[Statistics of temporal fan-out and edge feature homogeneity across datasets.]{\includegraphics[width=0.8\linewidth]{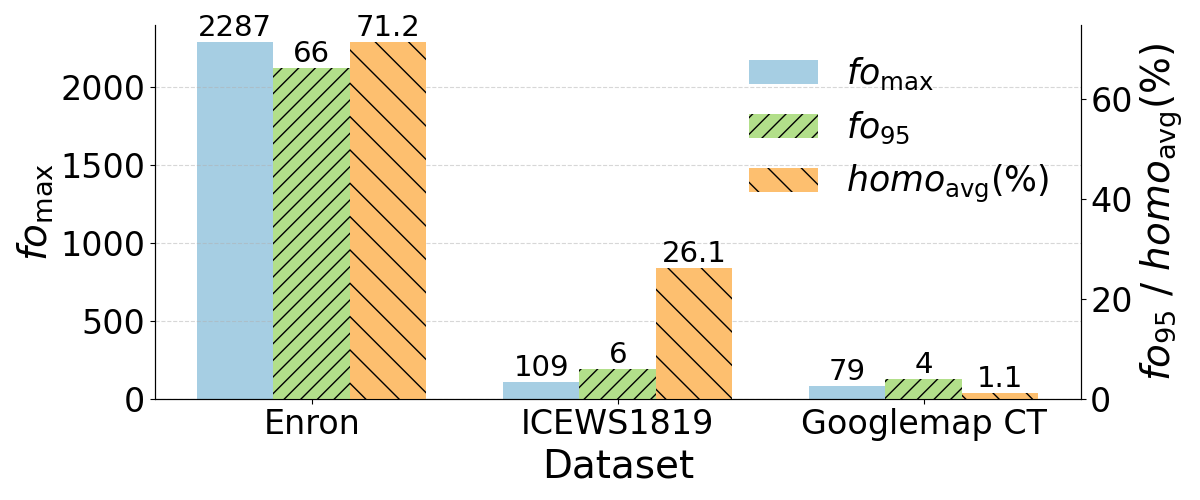}
    \label{fig:burst}}
    \vspace{-10pt}
    \subfigure[Example of temporal burstiness on the derived temporally biased random walks starting from node 120 in Enron.]{\includegraphics[width=0.8\linewidth]{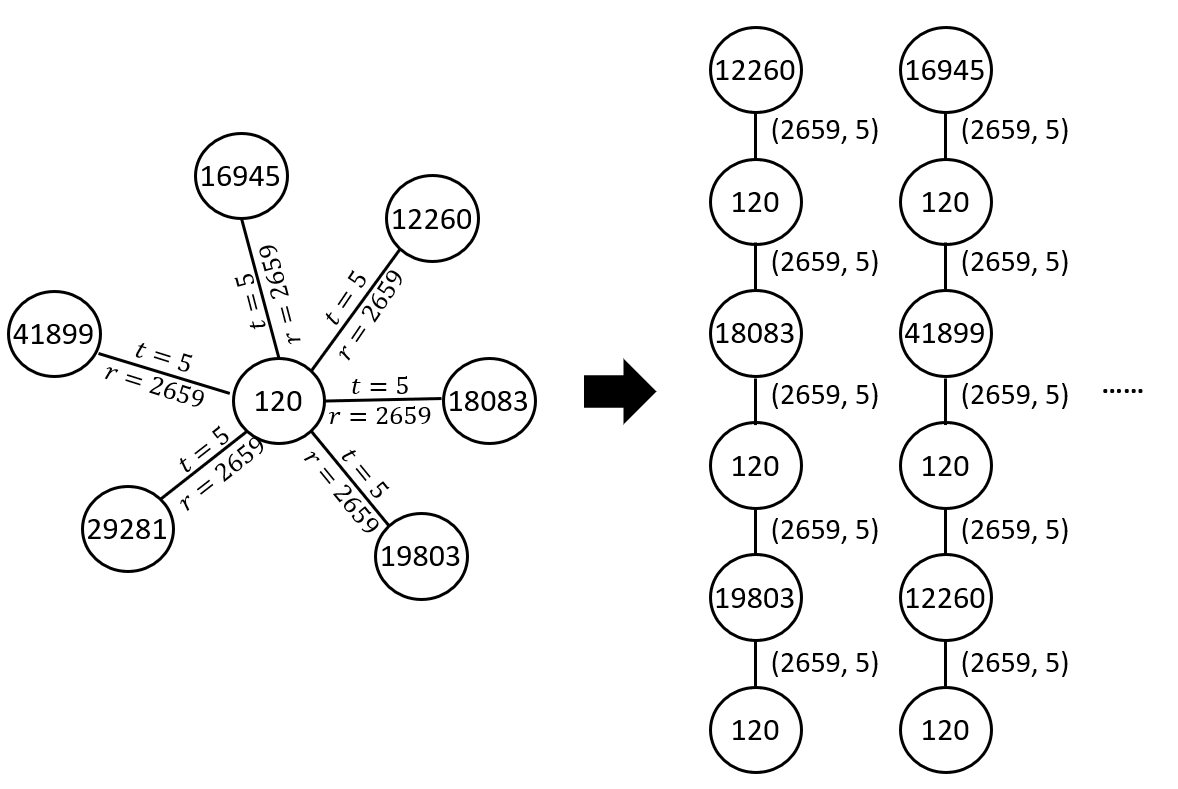}
    \label{fig:failcase}}
    \caption{\textbf{Failure Case Analysis in Enron,} which has large-scale homogeneous temporal burstiness, leading to highly similar sampled patterns and performance degradation of pretraining.  
    }
    \vspace{-15pt}
\end{figure}

% \textbf{Parameter Scaling.} Built on Transformer architecture, TGPM can naturally scale to larger model sizes by stacking more Transformer layers or extending to larger hidden dimension. As illustrated in Figure \ref{fig:scale}, increasing model parameters can consistently enhance performance on selected benchmarks. However, for datasets with abundant temporal dynamics like Googlemap CT, the performance gain from increased parameter is more promising. In contrast, for datasets with more concurrent edges, TGPM struggles to scale due to its disadvantage in processing temporal burstiness.

\textbf{Ablation Study.} We conduct two branches of ablation studies on Googlemap CT and Wikipedia: (1) Main ablation: we remove next time prediction (w/o NTP), long-term block masking (w/o LTM) and short-term block masking (w/o STM) respectively; (2) Alternative solution discussion: we first replace temporal random walk with chronologically monotonic causal path (w/ CP), and then replace short-/long-term masking with random masking (w/ RM). 
Table \ref{tab:ablation} presents the results of ablation studies. For our proposed components, NTP and short-term masking significantly contribute to evolving pattern mining, while the impact of long-term masking is domain-specific. This suggests that temporal evolving mechanisms and short-range dependencies are critical for understanding generalizable evolving patterns. Based on the results of alternative solutions, both causal paths and conventional random masking based generative pre-training consistently yield inferior performance, indicating the importance of pattern ordering relaxation and block-based dependency modeling. 

% \textbf{Computational Efficiency.} We evaluate computational efficiency by measuring average per-batch inference latency and peak GPU memory usage on Wikipedia. For supervised methods and self-supervised methods, we report the statistics during training and pre-training respectively. As shown in Figure \ref{fig:efficiency}, TGPM achieves competitive inference efficiency. Compared to self-supervised baselines, TGPM is substantially faster and more memory-efficient. These results demonstrate that TGPM achieves satisfying trade-off between predictive performance and inference efficiency.

\textbf{Failure Case Analysis.} In this section, we analyze the homogeneous temporal burstiness phenomenon in Enron and explain why TGPM fails in such situation. We design two metrics to measure the extent of temporal burstiness: \textit{(1) fan-out} $fo$ and \textit{(2) feature homogeneity score} $homo$, which measure the scale of temporal burstiness and the homogeneity of broadcasting behavior, respectively. Given a starting node $u$ and timestamp $t$, $fo$ and $homo$ can be defined as $fo(u,t) = |\{v|(u,v,t)\in\mathcal{E}\}|$, $homo(u,t)=\frac{
        \sum_{q} \binom{n_q(u,t)}{2}
        }{
        \binom{|\mathcal{S}_u^t|}{2}
        } $
where $q$ is the index of edge feature, $n_q(u,t)$ denotes the number of edges in $\mathcal{S}_u^t$ whose feature index equals $q$, and $\sum_{q}n_q(u,t)=|\mathcal{S}_u^t|$. Higher $fo$ and $homo$ refers to larger-scale homogeneous broadcasting behavior, or more concurrent edges with the same features. Figure \ref{fig:burst} reports the maximum $fo$, the 95th percentile of $fo$ and the average $homo$ ($fo_{\mathrm{max}},fo_{95}, homo_{\mathrm{avg}}$) of all the three benchmarks. Enron exhibits the largest scale of temporal burstiness, and the concurrent edges are homogeneous, corresponding to bulk emails in email network. As illustrated in Figure 5(b), this phenomenon produces highly similar interaction patches: half of sampled node features and all edge and temporal features are shared across sampled walks. As a result, masked patches can be trivially reconstructed from visible ones, leading to the collapse of MTM and NTP pre-training, preventing TGPM from learning meaningful temporal patterns. One promising direction is to aggregate homogeneous concurrent edges into meta-patterns before pre-training.

% As a result, the temporal motif illustrated in Figure \ref{fig:failcase} is very common in Enron, where many leaf nodes are connected with a central node by edges with exactly the same timestamp and features. The patterns sampled to describe the temporal context of the central node will be very similar: half of sampled node features and all the sampled temporal and edge features will be the same.
% Under such setting, our proposed MTM and NTP pre-training will easily lead to trivial solution, because the temporal random walks will degrade to basic random walks, making the masked patterns and next timestamp easier to be reconstructed from visible patterns. 
% As a result, it is difficult for TGPM to extract meaningful knowledge from large-scale homogeneous temporal burstiness. To overcome the challenge, a potential improvement is to aggregate large-scale homogeneous concurrent edges into a single meta-pattern, and model the correlations among meta-patterns. 

It is worth noting that while temporal burstiness is a well-studied property of real-world temporal networks~\cite{barabasi2005origin}, the failure mode identified here additionally requires large-scale feature homogeneity among concurrent edges, which is highly domain-specific. Real-world temporal graphs usually exhibit diverse edge features even during bursty periods, as concurrent interactions typically originate from independent user behaviors. Therefore, while this failure case represents a meaningful limitation of TGPM, it does not reflect a fundamental gap.

%% file: table/main_result.tex
\begin{table*}[!t]
  \centering
  \caption{\textbf{Performance and average ranking (A.R.) comparison over transductive and inductive temporal link prediction.} Results are reported on Enron, ICEWS1819 and Googlemap CT, with the best and sub-best performance highlighted in \textbf{Boldface} and \underline{Underline}. TGPM achieves the strongest A.R., with larger gains on datasets with rich temporal dynamics.}
  \resizebox{2.0\columnwidth}{!}{
    \begin{tabular}{l|ccc|ccc|c}
    \toprule
    \multirow{2}[4]{*}{\textbf{Methods}} & \multicolumn{3}{c|}{\textbf{\textit{Transductive Link Prediction}}} & \multicolumn{3}{c|}{\textbf{\textit{Inductive Link Prediction}}} & \multirow{2}[4]{*}{\textbf{A.R.}} \\
\cmidrule(lr){2-4} \cmidrule(lr){5-7}          & \textbf{Enron} & \textbf{ICEWS1819} & \textbf{Googlemap CT} & \textbf{Enron} & \textbf{ICEWS1819} & \textbf{Googlemap CT} &  \\
    \midrule
    TGAT  & 95.30±0.54 & 98.73±0.03 & 79.11±0.63 & 84.51±1.92 & 96.06±0.07 & \underline{73.41±1.07} & 3.8 \\
    GraphMixer & 95.31±0.19 & 98.68±0.01 & \underline{79.23±0.04} & 83.45±1.00 & 95.89±0.07 & 72.43±0.18 & 5.5 \\
    DyGFormer & 95.30±0.16 & 98.86±0.04 & 77.46±0.53 & \underline{86.42±0.34} & \textbf{96.60±0.16} & 70.97±0.36 & 3.8 \\
    \midrule
    PT-DGNN & \underline{95.55±0.13} & 98.81±0.07 & 77.35±0.76 & \textbf{86.64±0.89} & 96.07±0.31 & 71.50±1.31 & 4.0 \\
    DDGCL & 94.77±0.52 & 98.65±0.03 & 78.62±0.10 & 84.47±1.62 & 95.89±0.16 & 72.45±0.33 & 5.0 \\
    CPDG  & 93.97±0.62 & 97.54±0.04 & 76.52±0.58 & 80.23±1.36 & 92.04±0.08 & 70.16±0.98 & 7.7 \\
    \midrule
    TGPM w/o Pretrain & \textbf{96.06±0.15} & \underline{98.90±0.04} & 78.55±0.51 & 86.13±1.24 & 96.38±0.12 & 72.35±0.37 & \underline{3.2} \\
    TGPM  & 94.68±0.18 & \textbf{98.94±0.03} & \textbf{80.51±0.44} & 81.64±1.16 & \underline{96.39±0.18} & \textbf{74.32±0.34} & \textbf{3.0} \\
    \bottomrule
    \end{tabular}%
    }
  \label{tab:main}%
  \vspace{-10pt}
\end{table*}%

%% file: table/transfer.tex
\begin{table*}[!t]
\centering
\caption{\textbf{Comparison with self-supervised methods over cross-domain transferability.} TGPM achieves the strongest performance compared to prior works across all datasets.}
\label{tab:cross}
 \setlength{\tabcolsep}{4pt}
  \resizebox{2.0\columnwidth}{!}{
%     \begin{tabular}{l|cc|cc|c}
%     \toprule
%     \multicolumn{1}{c|}{\multirow{2}[4]{*}{\textbf{Methods}}} & \multicolumn{2}{c|}{\textbf{ICEWS1819}} & \multicolumn{2}{c|}{\textbf{Googlemap CT}} & \multicolumn{1}{c}{\multirow{2}[4]{*}{\textbf{A.R.}}} \\
% \cmidrule(lr){2-3}\cmidrule(lr){4-5}          & \textbf{Enron} & \textbf{Googlemap CT} & \textbf{Enron} & \textbf{ICEWS1819} & \multicolumn{1}{c}{} \\
%     \midrule
%     \multicolumn{6}{c}{\textit{\textbf{Transductive Link Prediction}}} \\
%     \midrule
%     PT-DGNN & 90.83±1.22 & \underline{53.79±6.25} & 65.79±5.58 & \underline{66.99±10.53} & \multicolumn{1}{c}{3.0} \\
%     DDGCL & \underline{91.40±1.30} & 52.42±4.93 & 66.66±13.68 & 56.71±17.33 & \multicolumn{1}{c}{\underline{2.8}} \\
%     CPDG  & 89.18±0.70 & 51.96±2.27 & \underline{75.74±2.58} & 66.12±6.45 & \multicolumn{1}{c}{3.3} \\
%     \midrule
%     TGPM  & \textbf{92.51±0.53} & \textbf{56.21±6.96} & \textbf{87.06±5.37} & \textbf{88.92±5.42} & \multicolumn{1}{c}{\textbf{1.0}} \\
%     \midrule
%     \multicolumn{6}{c}{\textit{\textbf{Inductive Link Prediction}}} \\
%     \midrule
%     PT-DGNN & 71.06±3.29 & 48.74±5.12 & 49.24±4.92 & 51.20±7.08 & 3.5 \\
%     DDGCL & \underline{72.97±3.70} & 49.38±3.68 & 58.42±5.78 & 46.05±10.41 & 3.0 \\
%     CPDG  & 65.89±2.98 & \underline{50.81±2.33} & \underline{67.72±5.04} & \underline{53.04±6.34} & \underline{2.5} \\
%     \midrule
%     TGPM  & \textbf{78.64±1.18} & \textbf{53.09±4.56} & \textbf{77.01±3.29} & \textbf{79.84±2.43} & \textbf{1.0} \\
%     \bottomrule
%     \end{tabular}%
   
\begin{tabular}{l|cc|cc|c|cc|cc|c}
\toprule
\multirow{3}{*}{\textbf{Methods}}
& \multicolumn{5}{c|}{\textbf{\textit{Transductive Link Prediction}}}
& \multicolumn{5}{c}{\textbf{\textit{Inductive Link Prediction}}} \\
\cmidrule(lr){2-6}\cmidrule(lr){7-11}
& \multicolumn{2}{c|}{\textbf{ICEWS1819}}
& \multicolumn{2}{c|}{\textbf{Googlemap CT}}
& \multirow{2}{*}{\textbf{A.R.}}
& \multicolumn{2}{c|}{\textbf{ICEWS1819}}
& \multicolumn{2}{c|}{\textbf{Googlemap CT}}
& \multirow{2}{*}{\textbf{A.R.}} \\
\cmidrule(lr){2-3}\cmidrule(lr){4-5}
\cmidrule(lr){7-8}\cmidrule(lr){9-10}
& \textbf{Enron}
& \textbf{Googlemap CT}
& \textbf{Enron}
& \textbf{ICEWS1819}
&
& \textbf{Enron}
& \textbf{Googlemap CT}
& \textbf{Enron}
& \textbf{ICEWS1819}
& \\
\midrule
PT-DGNN
& 90.83$\pm$1.22
& \underline{53.79$\pm$6.25}
& 65.79$\pm$5.58
& \underline{66.99$\pm$10.53}
& 3.0
& 71.06$\pm$3.29
& 48.74$\pm$5.12
& 49.24$\pm$4.92
& 51.20$\pm$7.08
& 3.5 \\

DDGCL
& \underline{91.40$\pm$1.30}
& 52.42$\pm$4.93
& 66.66$\pm$13.68
& 56.71$\pm$17.33
& \underline{2.8}
& \underline{72.97$\pm$3.70}
& 49.38$\pm$3.68
& 58.42$\pm$5.78
& 46.05$\pm$10.41
& 3.0 \\

CPDG
& 89.18$\pm$0.70
& 51.96$\pm$2.27
& \underline{75.74$\pm$2.58}
& 66.12$\pm$6.45
& 3.3
& 65.89$\pm$2.98
& \underline{50.81$\pm$2.33}
& \underline{67.72$\pm$5.04}
& \underline{53.04$\pm$6.34}
& \underline{2.5} \\

\midrule
TGPM
& \textbf{92.51$\pm$0.53}
& \textbf{56.21$\pm$6.96}
& \textbf{87.06$\pm$5.37}
& \textbf{88.92$\pm$5.42}
& \textbf{1.0}
& \textbf{78.64$\pm$1.18}
& \textbf{53.09$\pm$4.56}
& \textbf{77.01$\pm$3.29}
& \textbf{79.84$\pm$2.43}
& \textbf{1.0} \\
\bottomrule
\end{tabular}
    }
\end{table*}

%% file: section/related_work.tex
\section{Related Work}

\textbf{Temporal Graph Representation Learning.} Representation learning on temporal graphs has been widely studied~\cite{feng2025comprehensive,longa2023graph,kazemi2020representation}. 
Recently, continuous-time methods have emerged to directly learn on the whole event sequences of dynamic graphs with temporal random walks~\cite{nguyen2018continuous,wanginductive,jin2022neural}, temporal point process~\cite{trivedi2019dyrep,wen2022trend}, memory networks~\cite{kumar2019jodie,ma2020streaming,tgn_icml_grl2020,wang2021apan,luo2022neighborhood} and temporal graph neural networks~\cite{xu2020Inductive}. To efficiently learn from long-term interaction-wise dependencies, most state-of-the-art methods turn to sequential models, especially Transformer-based models~\cite{wang2021tcl,congwe,yu2023towards,wu2024feasibility,pan2025light,wu2025retrieval,peng2025tidformer}. Although these methods have shown empirical success, they are mostly task-centric methods suffering from limited receptive field, retrospective temporal modeling and inability to capture evolving semantics.
Our proposed TGPM can overcome these challenges.

\textbf{Temporal Graph Pre-Training.} Extended from static graph pre-training methods~\cite{hu2020strategies, zhao2021multi,zhao2023self,qian2022co,ju2022adaptive,liu2023graphssl,wang2023heterogeneous,wang2025training}, continuous-time temporal graph pre-training methods aim to capture inherent properties of temporal graphs via various self-supervised signals. Commonly used pre-training strategies include structural and temporal contrastive learning~\cite{bei2023cpdg,tian2021ddgcl}, masked attribute reconstruction~\cite{chen2022679}, curvature-aware self-contrastive learning~\cite{sun2022selfrgnn} and graph prompting~\cite{chen2024tigprompt,yunode,thapaliya2025semantic,wang2025can}. However, their focus on time-shifting invariant properties and low-level signal reconstruction misaligned with generalizable evolution pattern extraction. 
In contrast, TGPM introduces a pattern-centric pre-training framework targeting multi-scale temporal dynamics and prospective temporal modeling, which is able to learn from complex evolving patterns and unlock better generalizability.

\textbf{Graph Tokenization.} Graph pattern tokenization is a long-standing challenge due to the non-Euclidean nature of graphs. Early methods adopt graph kernels~\cite{yanardag2015dgk} and anonymous random walks~\cite{pmlr-v80-ivanov18a} to extract graph patterns. Recent approaches propose to tokenize graphs into sequences of substructures, such as subgraphs~\cite{he2023generalization}, multisets~\cite{baek2021accurate} and multi-hop neighborhood features~\cite{chen2023nagphormer,chen2024leveraging}. To further overcome task heterogeneity in graph learning, task-agnostic tokenization via message passing~\cite{Graph-Mamba,Chen22a,zhao2022from}, tree structures~\cite{wang2024gft,wang2025towards} and random walks~\cite{wangbeyond,wang2025generative} are proposed. However, these methods are inherently designed for static graphs, which emphasizes the learning of substructure correlations. In contrast, temporal graph learning focuses on evolution mechanism learning. 

%% file: section/conclusion.tex
\section{Conclusion}

We propose TGPM, a pattern-centric temporal graph learning framework that models generalizable temporal evolution mechanisms instead of task-specific correlations. By representing interactions as patches constructed via temporally biased random walks, TGPM captures evolving structural semantics and long-term dependencies. 
A Transformer backbone captures global invariances and local temporal patterns, aided by self-supervised objectives for learning multi-scale dynamics.
Extensive experiments demonstrate TGPM’s strong transferability and scalability. We also observe performance degradation under large-scale homogeneous temporal burstiness, indicating important directions for future work.

%% file: section/appendix.tex
\section{Additional Experimental Details}
\label{sec:implementation}

\subsection{Datasets}

For temporal link prediction, we evaluate all the selected methods on three semantically-enriched temporal graph benchmarks from the Dynamic Text-Attributed Graph Benchmark (DTGB)~\cite{zhang2024dtgb}: Enron, ICEWS1819, and Googlemap CT. These datasets cover diverse application domains and exhibit substantially different temporal characteristics, enabling a comprehensive evaluation of both in-domain performance and cross-domain generalization ability.
Enron is an email communication network where nodes represent email users and temporal edges correspond to email exchanges with timestamps. This dataset is characterized by strong temporal burstiness, where a large number of interactions may occur concurrently at the same timestamp, often sharing highly homogeneous edge features. 
ICEWS1819 is a temporal knowledge graph derived from the Integrated Crisis Early Warning System (ICEWS), covering international political events from 2018 to 2019. Nodes represent entities, and edges correspond to event interactions annotated with rich semantic features and timestamps. Compared to Enron, ICEWS1819 contains fewer concurrent interactions per timestamp and exhibits clearer temporal evolution patterns.
Googlemap CT is a temporal interaction graph collected from a recommender system, where nodes represent users or items and edges denote time-stamped behaviors. This dataset exhibits complex long-term temporal dynamics, making it well-suited for evaluating models’ ability to capture multi-scale temporal dependencies.

For temporal node classification, the evaluation involves two widely used benchmarks from \cite{kumar2019jodie}: Wikipedia and Reddit, which are both bipartite temporal interaction graphs recorded in Unix timestamps. In both datasets, nodes represent users and items. Items refer to Wikipedia pages and subreddits, while temporal edges correspond to page edits on Wikipedia and posts on Reddit respectively. The node classification task is to predict whether a user will be banned at a given time point.

\input{table/dataset}

\subsection{Baselines}

\textbf{Supervised methods.} TGAT~\cite{xu2020Inductive} employs temporal attention mechanisms to aggregate time-aware neighborhood information and models temporal interactions using continuous time encodings.
GraphMixer~\cite{congwe} replaces message passing paradigm with time-aware MLPMixer to efficiently capture temporal dependencies in interaction sequences.
DyGFormer~\cite{yu2023towards} is a Transformer-based model that represents historical interactions as token sequences and stacks self-attention layers to model long-range temporal dependencies.

\textbf{Self-supervised methods.} PT-DGNN~\cite{chen2022679} adopts masked attribute reconstruction as a pre-training strategy for dynamic graph neural networks.
DDGCL~\cite{tian2021ddgcl} performs self-supervised contrastive learning on temporal graphs by contrasting node representations across temporally adjacent views.
CPDG~\cite{bei2023cpdg} introduces contrastive pre-training based on temporal and structural augmentations to improve generalization.

\subsection{Implementation Details}

Most experiments are conducted on Linux servers equipped with four Nvidia A40 GPUs. The models are implemented by PyTorch 2.4.0 and PyTorch Cluster 2.1.2, with CUDA 12.1 and Python 3.9. The hyperparameter settings are summarized in Table \ref{tab:hyperparam}. For optimization, we use the AdamW optimizer with weight decay and set the number of epochs as 10 for pre-training and 20 for fine-tuning. Full fine-tuning is conducted after pre-training, while TGPM backbone and task-specific decoder use different learning rate, with 1e-5 and 1e-4 respectively. We use a two-layer multi-layer perceptron as the prediction head for temporal link prediction to take the concatenated representations of two input nodes and return the probability of the given link. Temporal link prediction experiments are conducted five times with different random seeds, while temporal node classification experiments are conducted for three times. 

\input{table/hyperparameter}

\section{Supplementary Experimental Results}

\begin{figure*}[!t]
\centering
\begin{minipage}[!t]{0.29\textwidth}
    \includegraphics[width=0.9\linewidth]{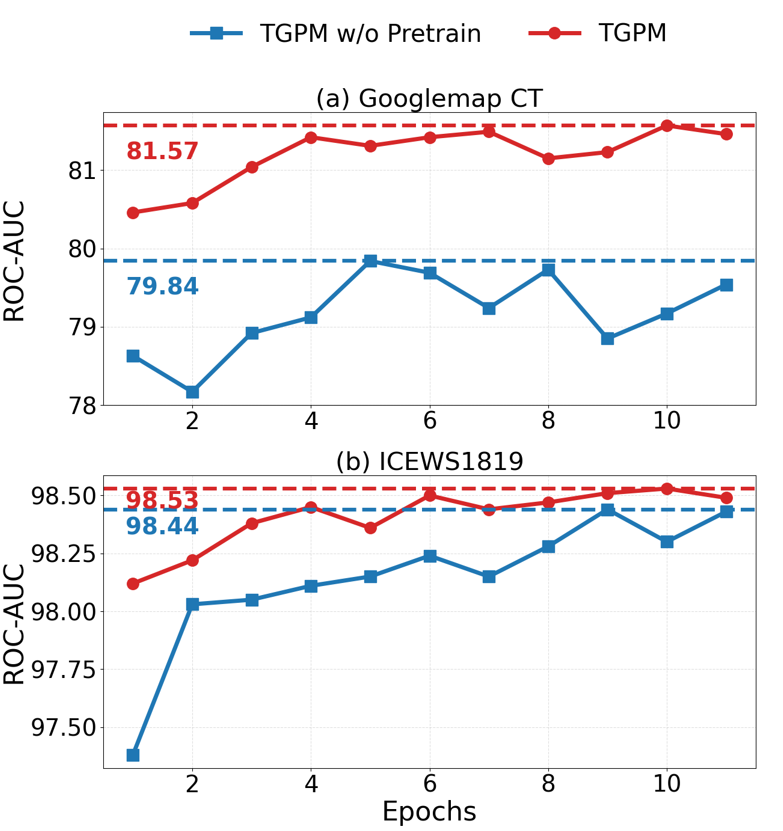}
\caption{\textbf{Convergence curve comparison between training from scratch and fine-tuning the pre-trained model. }
}
    \label{fig:converge}
\label{fig:efficiency}
\end{minipage}
\hfill
\begin{minipage}[!t]{0.33\textwidth}
    \includegraphics[width=0.9\linewidth]{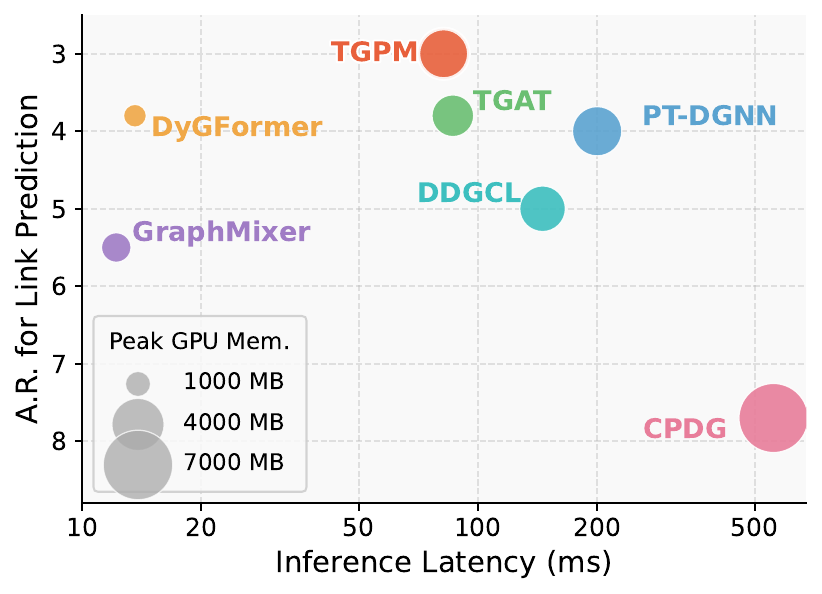}
\caption{\textbf{Computational Efficiency Analysis.} Sizes of points indicate the peak GPU memory usage. TGPM maintains competitive training efficiency, significantly more efficient than self-supervised baselines.
}
\label{fig:efficiency}
\end{minipage}
\hfill
\begin{minipage}[!t]{0.35\textwidth}
\centering
\includegraphics[width=0.9\linewidth]{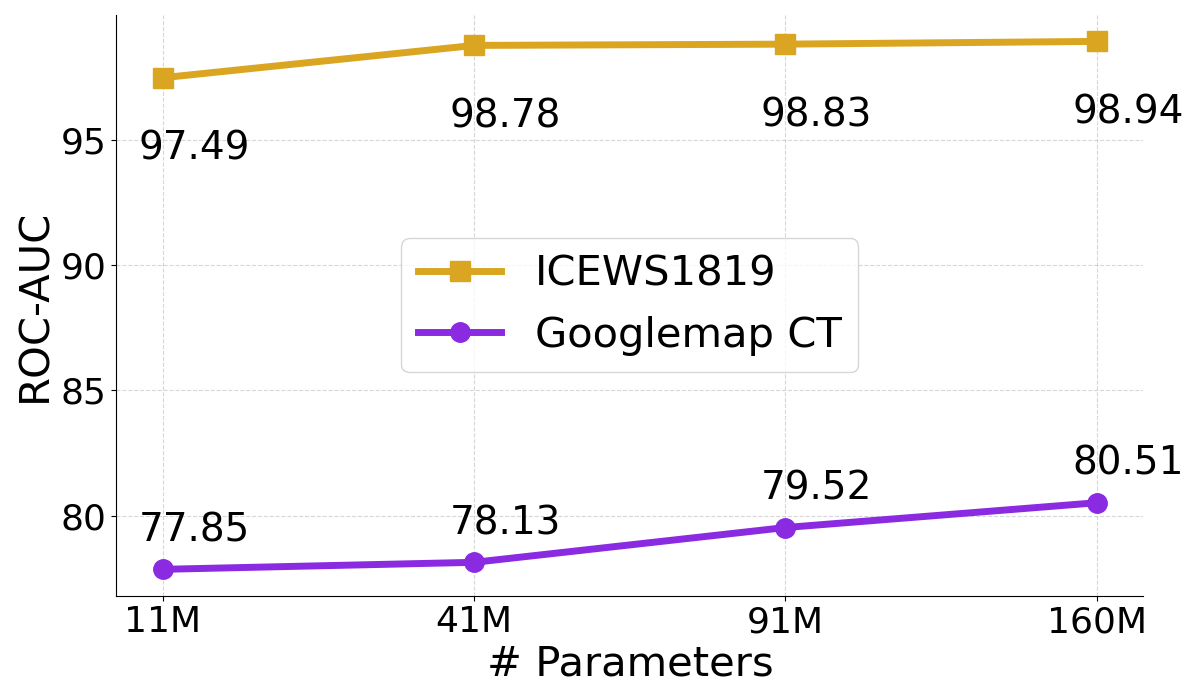}
\caption{\textbf{Parameter Scaling Analysis.} Increasing model parameters can consistently enhance performance on transductive settings, showing the great potential of TGPM acting as the scalable backbone in temporal graph learning. 
}
\label{fig:scale}
\end{minipage}
\vspace{-10pt}
\end{figure*}

\subsection{Convergence Curve Analysis}

Figure \ref{fig:converge} compares the training dynamics of TGPM trained from scratch versus those finetuned from a pre-trained TGPM checkpoint. We make comparisons over transductive temporal link prediction on Googlemap CT and ICEWS1819, where Googlemap CT has abundant temporal dynamics, while ICEWS1819 has fewer unique timestamps. Although pre-training facilitates performance improvement in both datasets, pre-training on datasets with more complex temporal dynamics like Googlemap CT leads to faster convergence and more significant performance gain, indicating that our proposed pre-training strategy can better capture complex temporal dependencies within rich temporal dynamics, improving in-domain generalization. As expected, accuracy improves with additional training epochs in both datasets.

\subsection{Computational Efficiency}

We evaluate computational efficiency by measuring average per-batch inference latency and peak GPU memory usage on Wikipedia. For supervised methods and self-supervised methods, we report the statistics during training and pre-training respectively. As shown in Figure \ref{fig:efficiency}, TGPM achieves competitive inference efficiency. Compared to self-supervised baselines, TGPM is substantially faster and more memory-efficient. These results demonstrate that TGPM achieves satisfying trade-off between predictive performance and inference efficiency.

\subsection{Parameter Scaling}

Built on Transformer architecture, TGPM can naturally scale to larger model sizes by stacking more Transformer layers or extending to larger hidden dimension. As illustrated in Figure \ref{fig:scale}, increasing model parameters can consistently enhance performance on selected benchmarks. However, for datasets with abundant temporal dynamics like Googlemap CT, the performance gain from increased parameter is more promising. In contrast, for datasets with more concurrent edges, TGPM struggles to scale due to its disadvantage in processing temporal burstiness.

\section{Proofs}
\label{sec:proof}

\input{section/math/3_1}

\input{section/math/3_2}

%% file: table/dataset.tex
\begin{table}[!h]
    \centering
    \caption{Dataset Statistics}
    \resizebox{0.7\linewidth}{!}{
        \begin{tabular}{c|cccc}
    \toprule
    \textbf{Dataset} & \textbf{Nodes} & \textbf{Edges} & \textbf{Unique Steps} & \textbf{Dim of Node/Edge} \\
    \midrule
    Enron & 42,711 & 797,907 & 1,006 & 768/768 \\
    ICEWS1819 & 31,796 & 1,100,071 & 730 & 768/768 \\
    Googlemap CT & 111,168 & 1,380,623 & 55,521 & 768/768 \\
    \midrule
    Wikipedia & 9,227 & 157,474 & 152,757 & --/172 \\
    Reddit & 10,984 & 672,447 & 669,065 & --/172 \\
    \bottomrule
    \end{tabular}%
    }
    \label{tab:dataset}%
\end{table}%

%% file: table/hyperparameter.tex
\begin{table}[htbp]
  \centering
  \caption{Hyperparameter Settings}
    \begin{tabular}{lc|lc}
    \toprule
    \textbf{Hyperparameter} & \textbf{Value} & \textbf{Hyperparameter} & \textbf{Value} \\
    \midrule
    Batch Size & 128    & Dropout & 0.1 \\
    Hidden Dimension & 768   & Weight Decay & 0.05 \\
    Time Encoding Dimension & 100   & Pre-Train Learning Rate & 0.0001 \\
    Number of Heads & 12    & Fine-Tune Learning Rate & 0.00001 \\
    Number of Encoder Layers & 2     & Pattern Size $L$ & 6 \\
    Number of Decoder Layers & 1     & Pattern Per Interaction $k$ & 8 \\
    EMA Momumtum & 0.99  & Masking Block Size $\mathcal{B}$ & [6, 24] \\
    EMA Update Every & 10    & Maximum Sampled Interaction $|\mathcal{S}_v^t|$  & 32 \\
    \bottomrule
    \end{tabular}%
  \label{tab:hyperparam}%
\end{table}%

%% file: section/math/3_1.tex
\subsection{Temporal Random Walk Expressiveness}
\label{sec:prop1}

\begin{definition}[Temporal Random Walk Expressiveness]
\label{def:walk_expressiveness}

% \zehong{Summarize the proof principle at the very top as a roadmap. Define temporal walks -> prove the expressiveness -> define reachability gap -> define disorder}

In this section, we formally establish the expressive advantage of non-monotonic temporal walks over monotonic walks. Our proof proceeds in four steps: (1) we define monotonic and non-monotonic temporal walks under the framework established in Section X; (2) we prove that non-monotonic walks are strictly more expressive by constructing a temporal graph where the inclusion is proper; (3) we quantify this expressiveness gap using the temporal reachability gap metric; and (4) we introduce temporal disorder as a measure to characterize the specific patterns captured by non-monotonic walks but missed by monotonic walks.

Given a temporal graph $\mathcal{G} = (\mathcal{V}, \mathcal{E}, \mathbf{X}, \mathbf{E})$ 
% \zehong{strictly follow the definition in the main paper} 
and anchor interaction $(v, v', t')$, we define:
\begin{itemize}
    \item \textbf{Monotonic temporal walk}: A node sequence $w = (v_0, v_1, \ldots, v_L)$ where $v_0 = v'$ and timestamps satisfy $\mathcal{T}(v_{i-1}, v_i) < \mathcal{T}(v_i, v_{i+1}) < t'$ for all valid $i$.
    \item \textbf{Non-monotonic temporal walk}: A node sequence $w = (v_0, v_1, \ldots, v_L)$ where $v_0 = v'$ and timestamps only satisfy $\mathcal{T}(v_i, v_{i+1}) < t'$ for all $i$.
\end{itemize}
Let $\mathcal{W}_{\text{mono}}^{t'}(v')$ and $\mathcal{W}_{\text{non}}^{t'}(v')$ denote the sets of reachable node sequences under monotonic and non-monotonic constraints, respectively.
\end{definition}

\begin{proposition}[Expressiveness of Non-Monotonic Temporal Walks]
\label{prop:expressiveness}
For any temporal graph $\mathcal{G}$ and anchor interaction $(v, v', t')$:
\begin{enumerate}
    \item[(i)] $\mathcal{W}_{\emph{mono}}^{t'}(v') \subseteq \mathcal{W}_{\emph{non}}^{t'}(v')$, i.e., non-monotonic walks are at least as expressive as monotonic walks.
    \item[(ii)] There exists a class of temporal graphs $\mathcal{G}^*$ such that $\mathcal{W}_{\emph{mono}}^{t'}(v') \subsetneq \mathcal{W}_{\emph{non}}^{t'}(v')$, i.e., the inclusion is strict.
    \item[(iii)] Non-monotonic walks can capture \emph{composite temporal motifs} that are fundamentally inaccessible to monotonic walks.
\end{enumerate}
\end{proposition}

\begin{proof}
\textbf{Part (i):} By definition, any monotonic temporal walk satisfies $\mathcal{T}(v_{i-1}, v_i) < \mathcal{T}(v_i, v_{i+1}) < t'$, which immediately implies $\mathcal{T}(v_i, v_{i+1}) < t'$. Thus, every monotonic walk is also a valid non-monotonic walk, establishing $\mathcal{W}_{\text{mono}}^{t'}(v') \subseteq \mathcal{W}_{\text{non}}^{t'}(v')$.

% \weixiang{check the example here}
\textbf{Part (ii):} We construct a temporal graph $\mathcal{G}^* = (\mathcal{V}^*, \mathcal{E}^*)$ where the inclusion is strict. Let $\mathcal{V}^* = \{A, B, C, D\}$ and consider the following temporal edges with anchor time $t' = 10$:
\begin{align*}
    \mathcal{E}^* = \{&(A, B, t=8), (B, C, t=3), (C, D, t=6)\}
\end{align*}

Starting from node $B$ (i.e., $v' = B$), consider the walk $w = (B, C, D)$:
\begin{itemize}
    \item \textbf{Non-monotonic}: Both edges satisfy the constraint $\mathcal{T}(e) < t' = 10$. Specifically, $\mathcal{T}(B,C) = 3 < 10$ and $\mathcal{T}(C,D) = 6 < 10$. Thus, $w \in \mathcal{W}_{\text{non}}^{t'}(B)$.
    \item \textbf{Monotonic}: The walk requires $\mathcal{T}(B,C) < \mathcal{T}(C,D)$, i.e., $3 < 6$, which holds. However, we also need the walk to originate from the anchor context. If we extend to walk $(A, B, C, D)$ starting from the anchor edge $(A, B, t=8)$, monotonicity requires $\mathcal{T}(A,B) < \mathcal{T}(B,C)$, i.e., $8 < 3$, which is false.
\end{itemize}

More directly, consider anchor $(A, B, t'=10)$ and the walk $w = (B, C)$. For monotonic walks rooted at this anchor, we require that subsequent edges have timestamps greater than the most recent edge in the path. Since we start from $B$ with implicit context timestamp $8$ (from edge $(A,B)$), reaching $C$ requires $\mathcal{T}(B,C) > 8$, but $\mathcal{T}(B,C) = 3 < 8$. Thus $(B, C) \notin \mathcal{W}_{\text{mono}}^{t'}(B)$ under anchor-aware monotonicity, while $(B, C) \in \mathcal{W}_{\text{non}}^{t'}(B)$.

\textbf{Part (iii):} We formally characterize the class of patterns accessible only to non-monotonic walks. Define a \emph{composite temporal motif} as a subgraph pattern where node relationships span multiple non-consecutive time periods.

\begin{definition}[Temporal Reachability Gap]
For anchor $(v, v', t')$, the \emph{temporal reachability gap} is defined as:
\begin{equation}
    \Delta_{\mathcal{G}}(v', t') = |\mathcal{W}_{\emph{non}}^{t'}(v')| - |\mathcal{W}_{\emph{mono}}^{t'}(v')|
\end{equation}
\end{definition}

The gap $\Delta_{\mathcal{G}}(v', t') > 0$ whenever the temporal graph contains \emph{temporally interleaved structures}, defined as node pairs $(u, w)$ reachable from $v'$ through paths where edge timestamps do not form a monotonic sequence. Such structures arise naturally in real-world temporal graphs:

\begin{itemize}
    \item \textbf{Rekindled relationships}: User $B$ interacted with $C$ in the past ($t=3$), then with $A$ recently ($t=8$). Monotonic walks from $B$ after the $(A,B)$ interaction cannot revisit the earlier $B$-$C$ relationship.
    \item \textbf{Periodic behaviors}: Seasonal patterns where interactions recur at non-consecutive intervals.
    \item \textbf{Multi-scale dynamics}: Long-term structural relationships that persist across short-term activity bursts.
\end{itemize}

% \zehong{What's the implications of the following proof? It seems to use another way to reinterpret the above things. What's the specific motivation of the following things?}

% \weixiang{The disorder metric directly motivates the exponential decay weighting in Eq. 1-2. Paths with higher disorder have more "time reversals" and should receive lower (but non-zero) sampling probability. The decay parameter $\lambda$ controls how much we penalize temporal disorder.}

% \zehong{Summarize the discussion here into the proof for motivating the use of disorder metric.}

To provide a finer-grained characterization of the expressiveness gap beyond mere cardinality, we introduce a quantitative measure of temporal non-monotonicity: Formally, let $\pi = (e_1, e_2, \ldots, e_k)$ be a path of temporal edges. Define the \emph{temporal disorder} of $\pi$ as:
\begin{equation}
    \text{disorder}(\pi) = \sum_{i=1}^{k-1} \mathbf{1}[\mathcal{T}(e_i) > \mathcal{T}(e_{i+1})]
\end{equation}

Monotonic walks are restricted to paths with $\text{disorder}(\pi) = 0$, while non-monotonic walks can traverse paths with $\text{disorder}(\pi) \geq 0$. The additional patterns captured by non-monotonic walks precisely correspond to paths with positive temporal disorder, which encode historically relevant but temporally non-sequential relationships. \qedhere
\end{proof}

\begin{remark}
The expressiveness gap has practical implications for temporal graph learning. Many real-world phenomena exhibit \emph{temporal non-monotonicity}: a user's current behavior may be better predicted by combining recent activity with historically distant but semantically relevant interactions, rather than strictly recent causal chains. Our temporally biased random walk with exponential decay weighting (Eq. 1-2) provides a principled mechanism to sample from $\mathcal{W}_{\emph{non}}^{t'}(v')$ while still prioritizing temporally proximate edges. 
% \zehong{Consider summarize these implications in the main paper.}
\end{remark}

%% file: section/math/3_2.tex
\subsection{Theoretical Analysis of Multi-Scale Masking}
\label{sec:prop2}

We provide an information-theoretic justification for our multi-scale masking strategy. The key insight is that \textbf{block size determines the range of temporal dependencies that can be effectively learned}: small blocks capture short-term patterns but lose long-term structure, while large blocks capture long-term patterns but suffer from information loss for short-term dependencies. This motivates a multi-scale approach.

\begin{definition}[$r$-Markov Dependency]
A token sequence $\bar{\mathbf{P}} = [\bar{\mathbf{p}}_1, \ldots, \bar{\mathbf{p}}_m]$ satisfies \emph{$r$-Markov dependency} if for all $i$:
\begin{equation}
    \bar{\mathbf{p}}_i \perp\!\!\!\perp \bar{\mathbf{P}}_{\{j: |j-i|>r\}} \mid \bar{\mathbf{P}}_{\{j: 0<|j-i|\leq r\}}.
\end{equation}
The \emph{dependency range} $r^*(\bar{\mathbf{P}})$ is the minimal such $r$. We say $\bar{\mathbf{P}}$ has \emph{multi-scale dependency} $(r_1, \ldots, r_K)$ if it decomposes as $\bar{\mathbf{p}}_i = f(\mathbf{u}_i^{(1)}, \ldots, \mathbf{u}_i^{(K)}, \bm{\epsilon}_i)$ where each $\mathbf{u}^{(k)}$ satisfies $r_k$-Markov dependency independently.
\end{definition}

\begin{definition}[Predictive Information]
% \zehong{From Mutual information}
We use $I(X; Y)$ to denote the mutual information between random variables $X$ and $Y$. For a contiguous block mask $M_b = \{a, \ldots, a+b-1\}$, the \emph{predictive information} is $\mathcal{I}(M_b) = I(\bar{\mathbf{P}}_{M_b}; \bar{\mathbf{P}}_{M_b^c})$. 

The \emph{$r$-local predictive information} is 
$$\mathcal{I}_r(M_b) = I(\bar{\mathbf{P}}_{M_b}; \bar{\mathbf{P}}_{M_b^c \cap \mathcal{N}_r(M_b)})$$
Where $\mathcal{N}_r(M_b) = \{j: \min_{i \in M_b}|j-i| \leq r\}$ is the $r$-neighborhood of $M_b$. This measures information accessible within range $r$—if $\mathcal{I}_r(M_b) < \mathcal{I}(M_b)$, then reconstruction requires looking beyond range $r$.

Finally, define the \emph{minimal reconstruction range}:
\begin{equation}
    R^*(M_b) = \min\{r : \mathcal{I}_r(M_b) = \mathcal{I}(M_b)\}
\end{equation}
This is the minimum context range needed for full predictive information.
\end{definition}

\begin{proposition}[Block Size Determines Learnable Dependency Scale]
\label{prop:block_info}
For a sequence with multi-scale dependency $(r_1, \ldots, r_K)$ where $r_1 < \cdots < r_K$, and block mask $M_b$ of size $b$:

\begin{enumerate}
    \item[(i)] \textbf{Block size lower-bounds required range}: $R^*(M_b) \geq b$. Large blocks require long-range context for reconstruction.
    
    \item[(ii)] \textbf{Scale mismatch causes information loss}:
    \begin{itemize}
        \item If $b \gg r_k$ (block too large for scale $k$): Dependencies at scale $r_k$ contribute only $O(r_k/b)$ fraction of reconstruction signal, causing \emph{short-term information loss}.
        \item If $b \ll r_k$ (block too small for scale $k$): The model cannot access context at range $r_k$ needed to capture scale-$k$ patterns, causing \emph{long-term information loss}.
        \item Optimal learning for scale $k$ occurs when $b \approx r_k$.
    \end{itemize}
    
    \item[(iii)] \textbf{Single block size is suboptimal}: No single block size $b$ can simultaneously satisfy $b \approx r_k$ for all $k \in \{1, \ldots, K\}$ when dependency scales are well-separated ($r_{k+1}/r_k \gg 1$).
\end{enumerate}
\end{proposition}

\begin{proof}
\textbf{Part (i)}: Consider the interior tokens of $M_b$, specifically those at positions $i$ with $\min_{j \in M_b^c}|i-j| = \lceil b/2 \rceil$ (center of the block). To reconstruct $\bar{\mathbf{p}}_i$, the model must access information from tokens at distance $\geq \lceil b/2 \rceil$. Even under 1-Markov dependency, information must propagate across this distance, requiring $R^*(M_b) \geq b/2$. For general $r$-Markov dependencies, the argument extends to $R^*(M_b) \geq b$.

\textbf{Part (ii)}: We analyze each mismatch scenario:

\textbf{Case 1: Block too large ($b \gg r_k$).} Consider scale-$k$ component $\mathbf{u}^{(k)}$ with $r_k$-Markov dependency. For an interior token $\bar{\mathbf{p}}_i$ at distance $d > r_k$ from block boundaries, the scale-$k$ information about $\bar{\mathbf{p}}_i$ is conditionally independent of visible context given the $r_k$-neighborhood. Specifically:
\begin{equation}
    I(\mathbf{u}_i^{(k)}; \bar{\mathbf{P}}_{M_b^c}) \approx 0 \quad \text{for } \min_{j \in M_b^c}|i-j| > r_k
\end{equation}
Only tokens within distance $r_k$ of boundaries (approximately $2r_k$ tokens total) have nonzero scale-$k$ information from context. Thus:
\begin{equation}
    \frac{\text{Scale-}k\text{ learnable signal}}{\text{Total block size}} \approx \frac{2r_k}{b} \to 0 \text{ as } b/r_k \to \infty
\end{equation}
This means \textbf{large blocks cannot effectively learn short-range dependencies}—most masked tokens have no learnable signal for fine-grained patterns.

\textbf{Case 2: Block too small ($b \ll r_k$).} For scale-$k$ patterns with dependency range $r_k$, full predictive information requires:
\begin{equation}
    \mathcal{I}(M_b) = \mathcal{I}_{r_k}(M_b) > \mathcal{I}_r(M_b) \quad \forall r < r_k
\end{equation}
However, if the model architecture or training procedure only accesses context within range $r_{\text{eff}} < r_k$ (e.g., due to limited attention span or locality bias), then the model cannot capture scale-$k$ information. Specifically, the learnable signal is:
\begin{equation}
    \mathcal{I}_{\text{learnable}}(M_b) \leq \mathcal{I}_{\min(r_{\text{eff}}, r_k)}(M_b) < \mathcal{I}(M_b) \quad \text{when } r_{\text{eff}} < r_k
\end{equation}
Moreover, even with sufficient architectural capacity, small blocks provide weak training signal for long-range patterns because the block size $b$ itself limits the ``observable distance" during masked prediction. This means \textbf{small blocks cannot effectively learn long-range dependencies}.

\textbf{Optimal matching}: When $b \approx r_k$, (1) interior tokens are within range $r_k$ of boundaries, enabling information flow, and (2) the block size matches the characteristic scale, maximizing learnable signal.

\textbf{Part (iii)}: Suppose dependency scales satisfy $r_{k+1}/r_k \geq c > 1$ for some constant $c$ (well-separated scales). For a single block size $b$:
\begin{itemize}
    \item If $b \approx r_j$ for some $j$, then for $k < j$: $b/r_k \geq c^{j-k} \gg 1$ (too large), and for $k > j$: $r_k/b \geq c^{k-j} \gg 1$ (too small).
    \item From part (ii), scales $k \neq j$ suffer information loss of order $O(1/c^{|j-k|})$.
\end{itemize}
Therefore, no single $b$ can efficiently capture all scales when $K > 1$ and scales are separated. \qedhere
\end{proof}

\begin{corollary}[Multi-Scale MTM Optimality]
\label{cor:multiscale}
To learn representations sufficient for prediction at all dependency scales $(r_1, \ldots, r_K)$:
\begin{enumerate}
    \item A set of block sizes $B = \{b_1, \ldots, b_K\}$ with $b_k \approx r_k$ is \textbf{necessary}.
    \item The combined objective $\mathcal{L}_{\text{MTM}} = \sum_{b \in B}\mathcal{L}_{\text{MTM}}(b)$ provides training signal at each scale, enabling the model to learn multi-scale dependencies that single-scale masking cannot capture.
\end{enumerate}
\end{corollary}

% \weixiang{remark insight pure language}

\begin{remark}
The proof establishes a complete information-theoretic characterization of the scale-matching problem in masked token modeling. The core mechanism is that block size $b$ creates an information bottleneck: Part (i) shows that reconstructing interior tokens requires context range $R^*(M_b) \geq b$, establishing a lower bound on the dependency scale accessible to a given block size; Part (ii) quantifies the cost of scale mismatch—when $b \gg r_k$, only $O(r_k/b)$ of tokens carry learnable signal for scale-$k$ patterns (short-term information loss), while when $b \ll r_k$, the model cannot access sufficient context to capture scale-$k$ dependencies (long-term information loss); Part (iii) proves this is unavoidable for any single block size when dependency scales are well-separated, as the mismatch penalty degrades exponentially at $O(1/c^{|j-k|})$ for scales $k \neq j$. Together, these results show that multi-scale MTM is not merely an empirical improvement but an information-theoretic necessity: each block size $b_k \approx r_k$ provides an independent learning channel for scale-$k$ dependencies, which is the only way to avoid fundamental information loss across the full spectrum of dependency scales present in sequential data.
\end{remark}